\documentclass[english,a4paper,12pt]{article}

\usepackage{amsxtra, amsfonts, amssymb, amstext, amsmath, mathtools}%,

\usepackage{booktabs}
\usepackage{amsthm}

\usepackage{nicefrac}
\usepackage{xspace}
\usepackage{url}
\usepackage{graphics,xcolor}
\usepackage[ruled,vlined,linesnumbered]{algorithm2e}
\usepackage{wrapfig}

\usepackage{graphicx}
\usepackage{dsfont}
\usepackage[shortcuts]{extdash}
\usepackage{subcaption}

\graphicspath{{images/}}

\newtheorem{theorem}{Theorem}
\newtheorem{lemma}[theorem]{Lemma}

\newcommand{\oea}{\mbox{${(1 + 1)}$~EA}\xspace}

\newcommand{\ooea}{\oea}

\newcommand{\mpoea}{\mbox{${(\mu+1)}$~EA}\xspace}

\newcommand{\moga}{\mbox{${(\mu+1)}$~GA}\xspace}

\newcommand{\opllga}{\mbox{${(1+(\lambda,\lambda))}$~GA}\xspace}
\newcommand{\ollga}{\opllga}

\newcommand{\mga}{\mbox{${(\mu+1)}$~GA}\xspace}
\newcommand*{\yIndividual}{$y$-individual\xspace}
\newcommand*{\yIndividuals}{$y$-individuals\xspace}
\newcommand*{\nonyIndividual}{non-$y$-individual\xspace}
\newcommand*{\nonyIndividuals}{non-$y$-individuals\xspace}

\newcommand{\onemax}{\textsc{OneMax}\xspace}

\newcommand{\jump}{\textsc{Jump}\xspace}

\newcommand{\jumpk}{\textsc{Jump}_k}

\newcommand{\R}{\ensuremath{\mathbb{R}}}

\newcommand{\N}{\ensuremath{\mathbb{N}}} % ohne Null!!!

\newcommand{\bbone}{{\mathds{1}}}

\newcommand{\calF}{\ensuremath{\mathcal{F}}}

\DeclareMathOperator{\E}{\mathds{E}}
\newcommand{\eps}{\varepsilon}

\let\originalleft\left
\let\originalright\right
\renewcommand{\left}{\mathopen{}\mathclose\bgroup\originalleft}
\renewcommand{\right}{\aftergroup\egroup\originalright}

\usepackage{hyperref}

\title{Lasting Diversity and Superior Runtime Guarantees for the $(\mu+1)$ Genetic Algorithm}

\author{Benjamin Doerr$^{1}$ \and Aymen Echarghaoui$^{2}$ \and Mohammed Jamal$^{2}$ \and Martin~S. Krejca$^{1}$}

\date
{
        \small $^{1}$Laboratoire d'Informatique (LIX), CNRS, École Polytechnique,\\
        \small Institut Polytechnique de Paris, Palaiseau, France\\
        \small $^{2}$École Polytechnique, Institut Polytechnique de Paris, Palaiseau, France
}

\begin{document}

\maketitle

\sloppy{
\begin{abstract}
  Most evolutionary algorithms (EAs) used in practice employ crossover. In contrast, only for few and mostly artificial examples a runtime advantage from crossover could be proven with mathematical means. The most convincing such result shows that the $(\mu+1)$ genetic algorithm~(GA) with population size $\mu=O(n)$
  optimizes jump functions with gap size $k \ge 3$ in time $O(n^k / \mu + n^{k-1}\log n)$, beating the  $\Theta(n^k)$ runtime of many mutation-based EAs.
  This result builds on a proof that the GA occasionally and then for an expected number of $\Omega(\mu^2)$ iterations has a population that is not dominated by a single genotype.

  In this work, we show that this diversity persist with high probability for a time exponential in $\mu$ (instead of quadratic). From this better understanding of the population diversity, we obtain stronger runtime guarantees, among them the statement that for all $c\ln(n)\le\mu \le  n/\log n$, with~$c$ a suitable constant, the runtime of the $(\mu+1)$~GA on $\jump_k$, with $k \ge 3$, is $O(n^{k-1})$. Consequently, already with logarithmic population sizes, the GA gains a speed-up of order $\Omega(n)$ from crossover.
\end{abstract}

\section{Introduction}

In the vast majority of applications of evolutionary algorithms (EAs), crossover is employed. This is natural since EAs are inspired by natural evolution. However, since the ultimate goal of EAs is not to mimic evolution but to compute good solutions for a given problem, one could wonder if there is a real justification for the use of crossover, or, asking the question in a less absolute manner, the questions are for which problems crossover can speed up evolutionary algorithms, what are the reasons for the possible performance gains through crossover, and what is the best way to employ crossover (parameters, crossover operators, etc.).

With their fundamental nature, these are all question one could try to answer via theoretical means, e.g., via mathematical runtime analyses~\cite{AugerD11,DoerrN20,Jansen13,NeumannW10}, which have given many important insights and explanations in the past. Unfortunately, when it comes to understanding crossover, the mathematical runtime analysis area was not very successful (but we note that other attempts to understand crossover, e.g., the building block hypothesis~\cite{Holland75} also struggled to explain crossover~\cite{MitchellHF93}).

The main difficulty towards understanding crossover via mathematical tools are the usually complicated population dynamics. We note that even when only regarding mutation-based algorithms, still the majority of runtime analyses regards algorithms with trivial populations such as \emph{randomized local search} or the \oea~\cite{DoerrN20}. The particular difficulty when adding crossover is that we need to understand the diversity of the population. Crossover can only be really effective when sufficiently different, good solutions are available. That this is the core problem is easily visible from the existing runtime analyses of crossover-based algorithms. We refer to Section~\ref{sec:previous} for a more detailed discussion and give only few examples here. That diversity, more precisely, arriving at a state with a diverse population, is the crucial ingredient for profiting from crossover was already discussed very clearly in the first runtime analysis of a crossover-based algorithm~\cite{JansenW02}, which analyzes how the $(\mu+1)$ genetic algorithm (GA) optimizes the \jump benchmark. Since crossover can also contribute towards losing diversity, the positive results shown in~\cite{JansenW02} are only valid for an unrealistically small crossover rate. Doerr, Happ, and Klein~\cite{DoerrHK12}, without adding uncommon assumptions, show an advantage of crossover when solving the all-pairs-shortest-path problem, but it is clear that this problem, asking for a short path between any two vertices, just in the problem definition contains a strong diversity mechanism that prevents the population from converging to a single genotype.

The maybe most convincing work proving an advantage from crossover by Dang et~al.~\cite{DangFKKLOSS18}, where again the \moga optimizing \jump functions is regarded.
Interestingly, without particular modifications of the algorithm, the authors manage to show that it repeatedly arrives at diverse populations and keeps this diversity for a moderate time. More precisely, when optimizing a jump function with gap size~$k$, then after an expected time of $O(n\mu / k)$ iterations, the algorithm reaches a state where the largest genotype class in the population contains at most $\mu/2$ individuals.
This diversity of the largest genotype class missing a constant fraction of the population stays for an expected $O(\mu^2)$ iterations. Given these numbers, the best runtime results are obtained for a relatively large population size of $\Theta(\mu)$, where a runtime gain of roughly a factor of $n / \log n$ can be proven. More precisely, for a general population size $\mu \le \kappa n$, with~$\kappa$ a sufficiently small constant, it was shown that  the \moga optimizes jump functions with gap size $k \ge 3$ in time $O(n^k / \mu + n^{k-1}\log n)$.

The key proof argument of Dang et~al.~\cite{DangFKKLOSS18} is that the size of the largest genotype class, once it is below $\mu$, is at least as likely to decrease as to increase. This give a behavior resembling an unbiased random walk in the interval $[\mu/2..\mu]$, and this is the reason why the diversity is shown to persist for $O(\mu^2)$ iterations.

In this work, we analyze the random process describing the largest genotype class of the \mga on \jump more carefully and show that once it is by a constant factor below~$\mu$, we do not have an unbiased random walk but a constant drift towards smaller values (Lemma~\ref{lem:probound}). Such a constant drift in, say, the interval $[0.5\mu,0.75\mu]$, allows to argue that this diversity remains for a number of iterations exponential in~$\mu$. Consequently, already for much smaller population sizes, we prove a most longer-lasting diversity than in the previous work.

This result is interesting in its own right as it shows that diverse parent populations are much easier to obtain than what previous works suggest. For the particular problem of how fast the \mga optimizes \jump functions, we obtain a stronger advantage from crossover (Theorem~\ref{thm:runtime}), namely by a factor of $\Omega(n)$, and this for all population sizes that are at least $C \ln n$, $C$ a suitable constant.  For smaller values of $\mu$, we still obtain a significant speed-up from crossover, namely a runtime of $O(e^{-\Omega(\mu)} n^k)$, again a good improvement over the $O(n^k / \mu)$ guarantee from Dang et~al.~\cite{DangFKKLOSS18}.

Overall, our results give some hope that also basic evolutionary algorithms without greater adjustment are able to reach and then keep a diverse population for a long time. This is clearly good news for the use of crossover but might also yield other advantages, like a more effective exploration of the search space.

\section{Previous Work}\label{sec:previous}
Since the first rigorous result showing the benefit of crossover more than two decades~\cite{JansenW99} ago, a variety of results for genetic algorithms (GAs), i.e., evolutionary algorithms that employ crossover, has been proven.
Because crossover introduces many dependencies into the population of the GA in question, its analysis is notoriously difficult.
Hence, most of the existing results make additional, restrictive assumptions.

\textbf{Non-standard parameter values.}
The analysis of GAs was initiated by Jansen and Wegener~\cite{JansenW99}, who considered jump functions of gap size~$k$.
They prove that the \oea requires in expectation $\Theta(n^k)$ function evaluations before it creates the optimum, whereas a \mga with crossover rate~$p_c$ only requires $O(\mu n^2 k^3 + 4^k / p_c)$ evaluations.
However, in order to prove their result, the authors assume that $p_c = O(1 / (kn))$, relying on a long sequence of preferable mutations before crossover is capable to produce the optimum more easily than mutation.
In this setting, crossover is only seen as useful once a sufficient diversity is achieved.

This result was later improved by Kötzing, Sudholt, and Theile~\cite{KotzingST11}, but the authors still require that $p_c = O(k / n)$, as their proof relies on the same technique as that of Jansen and Wegener~\cite{JansenW99}.

\textbf{Constant-factor improvements.}
For the classical \onemax problem, Sudholt~\cite{Sudholt12} proves that a special variant of the \mga has a constant-factor speed-up compared to any variant of the \ooea with mutation rates between~$2^{-n / 3}$ and $(\sqrt{n} \ln(n))^{-1}$.
As the algorithm only applies crossover to the best individuals in the population, the $(2+1)$~GA uses the best population size.

Corus and Oliveto~\cite{CorusO18tec} remove the elitist selection for crossover from the analysis of Sudholt~\cite{Sudholt12} and analyze the typical \mga, with a crossover rate of~$1$.
The authors prove that speed-ups of up to $25$\,\% are possible when compared to the expected runtime of the \ooea.
In a subsequent work, the authors~\cite{CorusO20} improve this result by showing speed-ups of at least $60$\,\% compared to not using crossover.
These new results also prove runtime bounds that reduce with the population size, thus showing a potential benefit of a larger population.

Recently, Oliveto, Sudholt, and Witt~\cite{OlivetoSW22} prove a lower bound for the $(2+1)$~GA on \onemax that matches the previous results~\cite{CorusO20}.
Together, these results show that having a larger population than~$2$ is truly beneficial.

\textbf{Diversity mechanisms.}
Another common way to analyze crossover is to employ mechanisms that aim at increasing the diversity in the population of the GA.
Depending on the mechanism and on the problem, such approaches can be very fruitful.
However, the respective results rather showcase, similar to the non-standard parameter values, how well crossover can exploit a diverse population.
It is not clearly evident to what extent crossover by itself is already useful.

Oliveto, He, and Yao~\cite{OlivetoHY08} consider the vertex cover problem and analyze the \mpoea (without crossover) as well as the $(\mu + 1)$ randomized local search (RLS) with crossover, both employing deterministic crowding.
This mechanism guarantees that if a newly created individual is at least as good as its parent(s), one of the parents is removed.
The paper shows that just using deterministic crowding without crossover is already sufficient for finding minimal covers on bipartite graphs.
Hence, the results rather show the benefit of the diversity mechanism.

Lehre and Yao~\cite{LehreY11} also consider deterministic crowding, combined with a variant of the \mga (called SSGA).
They show that using crossover with constant crossover rate improves the expected runtime of the $(\mu + 1)$~SSGA from at least exponential to quadratic in the problem size.
Although this is an impressive speed-up, the considered algorithm uses a crossover that exchanges exactly one component among two parents, thus making rather local changes, which is uncommon for crossover.

Neumann et~al.~\cite{NeumannORS11} consider the single-receiver island model with $1$-point crossover.
They consider a problem for which a variant of the \mga (for populations of order at most $n / \log^3(n)$) has for each crossover rate at least an exponential runtime, with high probability.
In contrast, they prove that the island model optimizes this problem with high probability efficiently.
This result rather highlights the benefits of the island model in this scenario than of crossover.

Dang et~al.~\cite{DangFKKLOSS16} analyze the expected runtime of the \mga optimizing jump functions, for a variety of different diversity mechanisms.
The runtime is always better (for appropriate parameters) than that of the mutation-only \ooea.
As before, these results heavily rely on the diversity mechanism for crossover to be useful.

Sutton~\cite{Sutton21} proves that a multi-start variant of the \mga optimizes the closest-string problem in randomized fixed-parameter tractable (FPT) time.
The article also shows that if crossover is removed, there are instances that are not solved in FPT time.

\textbf{Problem-specific knowledge.}
Doerr, Happ, and Klein~\cite{DoerrHK12} analyze the impact of crossover for the all-pairs-shortest-path problem, which is a non-artificial, combinatorial problem.
The authors prove that the algorithm variant with crossover improves the expected runtime by a factor of $\Theta(\sqrt{n / \log(n)})$.
In order to prevent infeasible solutions from entering the population, the algorithm employs a mechanism that detects infeasible solutions and discards them.
This work was later improved by Doerr et~al.~\cite{DoerrJKNT13} by utilizing better crossover operators.

\textbf{Non-standard crossover operator.}
A decade ago, Doerr, Doerr, and Ebel~\cite{DoerrDE13,DoerrDE15} proposed the \ollga---a GA that introduces a novel crossover operator that is aimed to exploit information learned from good \emph{and} bad samples.
Doerr and Doerr~\cite{DoerrD15tight} prove that the \ollga has an expected runtime on \onemax that is improved by a factor of order $\sqrt{\ln(n)}$ compared to many other evolutionary algorithms.

Since then, the \ollga has been studied on different problems~\cite{AntipovDK19foga,AntipovDK20,AntipovDK22,DoerrNS17}.
Notably, Antipov, Doerr, and Karavaev~\cite{AntipovDK22} prove that the \ollga optimizes jump functions with gap size~$k$, with optimal parameters (depending on~$k$), in the order of $(n / k)^{k / 2} e^{\Theta(k)}$, which is faster by a factor of $(n k)^{-k / 2} e^{\Theta(k)}$ compared to the traditional \ooea.

Other articles determine optimal parameter settings of the \ollga and propose how to choose them dynamically~\cite{AntipovD20ppsn,Doerr16,DoerrD18}.
Especially, Antipov and Doerr~\cite{AntipovD20ppsn} show for jump functions with gap size~$k$ that choosing the parameters of the \ollga according to a heavy-tailed distribution (independent of~$k$) achieves a performance close to the best instance-specific parameter choice.

\textbf{Multi-objective optimization.}
Although it is not the focus of this article and not a restriction, we would like to mention that crossover has also been studied for multi-objective problems~\cite{HuangZCH19,NeumannT10,QianYZ13}.
Especially recently, the most-widely used EA for multi-objective optimization, NSGA\=/II~\cite{DebPAM02}, which employs crossover, has been studied~\cite{BianQ22,DangOSS23,DoerrQ23LB}.

\textbf{Standard algorithm with standard parameters.}
To the best of our knowledge, the only result that considers the \mga with typical parameters is the article by Dang et~al.~\cite{DangFKKLOSS18}.
For jump functions with gap size~$k$, the authors prove (for optimal parameters, including a slightly higher mutation rate than normal) an expected runtime of up to $O(n^{k - 1})$, which is better by a factor of~$n$ than the expected runtime when employing no crossover.
The analysis tracks the diversity of the population, showing that there are phases lasting about $\mu^2$ iterations in expectation in which a sufficient diversity in the population emerges naturally.
During such a phase, the probability to create the optimum via crossover and mutation is better by a factor of about~$n$ compared to only mutation.
The analysis further reveals that crossover is \emph{not} detrimental for creating the desired diversity, thus, providing a convincing argument for why crossover is typical favorable in practice.

\section{Preliminaries}

We introduce the notation used throughout this paper, the algorithm and function that we consider in our analysis (Section~\ref{sec:notation}), as well as the stochastic tools we use for deriving our results (Section~\ref{sec:tools}).

\subsection{Mathematical Notation}
\label{sec:notation}

Let $\N$ denote the set of all natural numbers (including $0$), and let $\R$ be the set of real numbers. For $m,n \in \N$, let $[m..n] = [m,n] \cap \N$, that is, the discrete interval from~$m$ to~$n$. Further, we define the special case $[n] \coloneqq [1..n]$. We consider pseudo-Boolean maximization, that is, for a given $n \in \N_{\geq 1}$, we aim to find a global maximum for a given $n$-dimensional objective function $f\colon \{0,1\}^n \to \R$ (a \emph{fitness function}), which maps bit strings to reals. We call elements from $\{0,1\}^n$ \emph{individuals} and their objective-function value \emph{fitness}. Throughout this paper, whenever we use big-O notation, we assume that the statement is asymptotic in the problem dimension~$n$.

Given a random variable~$X$, a $\sigma$-algebra~$\calF$, and an event~$A$ with $\Pr[A] > 0$, let $\E[X \mid \calF; A] = \E[X \cdot \bbone_A \mid \calF] / \Pr[A]$.
We make use of this notation when we condition on a $\sigma$-algebra but also make case distinctions via events.

For two \emph{individuals} $x,y \in \{0,1\}^n$, we define their \emph{Hamming distance} as the number of positions that they differ in. Further, let $|x|_{1}$ denote the number of $1$s in a bit-string~$x$.

Let $n \in \N_{\geq 1}$, and let $k \in [n]$.
We consider the $\jumpk \colon \{0,1\}^n \to \N$ fitness function, which is defined as
\begin{equation}
  \label{eq:jumpDefinition}
  \jumpk\colon x \mapsto \begin{cases}
    k + |x|_{1} & \text{if $|x|_{1}=n$ or $|x|_{1} \leq n-k$} , \\
    n-|x|_{1}    & \text{otherwise.}
  \end{cases}
\end{equation}
Whenever we mention $\jumpk$ in this paper, we assume that its parameters~$n$ and~$k$ are implicitly given.

The function value of $\jumpk$ increases with the number of~$1$s in the individual until the \emph{plateau} of local optima is reached, which consists of all points with exactly $n-k$ $1$s. However, the only global optimum of $\jumpk$ is the all-$1$s string. Between the plateau and the global optimum, that is, for all individuals with more than $n - k$ but less than~$n$ $1$s, there is a gap of length $k - 1$ of fitness worse than for any individual outside of this area.
Elitist EAs that only have solutions on the plateau thus have to change~$k$ bits in a single iteration in order to get from the plateau to the global optimum (also referred to as a \emph{jump}).

We consider the \mga (Algorithm~\ref{alg:gea}) with \emph{population size} $\mu \in \N_{\geq 2}$, \emph{crossover rate} $p_c \in [0,1]$, and \emph{mutation rate} $p_m \in [0,1]$.
The algorithm maintains a multiset (the \emph{population}) of~$\mu$ individuals, which is initialized with uniform samples from the search space and afterward updated iteratively.
In each iteration, a new individual is created and potentially replaces an individual in the population in the following way, where all random choices are independent:
With probability~$p_c$, two individuals from the current population are selected uniformly at random (the \emph{parents}).
Then, a new individual (the \emph{offspring}) is created by \emph{uniform crossover} of the parents, that is, each bit of the offspring is chosen uniformly at random from the respective positions of the parents.
Afterward, \emph{standard bit mutation} with mutation rate~$p_m$ is applied to the offspring, that is, each of its bits is flipped independently with probability $p_m$.
Otherwise, with probability $1 - p_c$, a parent is selected uniformly at random from the current population and produces an offspring only via standard bit mutation.
Last, in any case, an individual with the lowest fitness among the current population and the offspring is removed, breaking ties uniformly at random.
The resulting multiset forms the population for the next iteration.
We define the runtime of the \mga as the number of fitness function evaluations until an optimal solution is evaluated for the first time.

\begin{algorithm}[t]
  \caption{\label{alg:gea} The \mga with parameters $\mu \in \N_{\geq 2}$ and $p_c, p_m \in [0, 1]$, maximizing fitness function~$f$}
  $t \gets 0$\;
  $P^{(0)} \gets \mu$ individuals, uniformly at random from $\{0,1\}^n$\;
  \While( \texttt{// iteration}~$t$)
  {\emph{termination criterion met}}
  {
    $p \gets$ value from $[0, 1)$ uniformly at random\;
    \If{$p < p_c$}
    {
      $x^{(t)}, y^{(t)} \gets$ two individuals from~$P^{(t)}$ chosen uniformly at random (with replacement)\;
      $\widetilde{z}^{(t)} \gets$ new individual created by uniform crossover of~$x^{(t)}$ and~$y^{(t)}$\;
      $z^{(t)} \gets \widetilde{z}^{(t)}$ augmented by standard bit mutation\;
    }
    \Else
    {
      $m^{(t)} \gets$ copy of an individual from~$P^{(t)}$ chosen uniformly at random\;
      $z^{(t)} \gets m^{(t)}$ augmented by standard bit mutation\;
    }
    $\ell^{(t)} \gets$ individual with the lowest fitness from $P^{(t)} \cup \{z^{(t)}\}$, breaking ties uniformly at random\;
    $P^{(t + 1)} \gets (P^{(t)} \cup \{z^{(t)}\}) \setminus \{\ell^{(t)}\}$\;
    $t \gets t + 1$\;
  }
\end{algorithm}

\subsection{Tools for Our Analyses}\label{sec:tools}

Our analysis of the \mga on $\jumpk$ (Section~\ref{sec:results}) follows the one by Dang et~al.~\cite{DangFKKLOSS18} and splits the runtime of the algorithm into the following two phases:
The first phase considers the expected number of iterations until the algorithm finds the optimum or until the entire population is on the plateau.
The second phase carefully analyses the diversity of the population on the plateau until the optimum is created.
For the first phase, we use the result by Dang et~al.~\cite{DangFKKLOSS18} for the general setting of $p_c=\Omega(1)$.

\begin{theorem}[{\cite[Lemma~$1$]{DangFKKLOSS18}}]\label{thm:timeplateau}
  Consider the \mga with $p_c=\Omega(1)$ and $p_m=\Theta(\frac{1}{n})$ optimizing $\jumpk$ with $k= o(n)$.
  Then the expected time until either the optimum is found or the entire population is on the plateau is $O(n\sqrt{k}(\mu \log(\mu)+ \log(n)))$.
\end{theorem}

For the second phase, we eventually aim to create the optimum of $\jumpk$ via crossover (followed by mutation) of two individuals with a Hamming distance of~$2$.
The following theorem provides the probability of constructing the optimum in the general case.
It is a slight reformulation of {\cite[Lemma~$3$]{DangFKKLOSS18}} (from mutation rate~$\frac{1}{n}$ to rate~$\Theta(\frac{1}{n})$).
For the sake of completeness, we still provide a proof.

\begin{lemma}[{\cite[Lemma~$2$]{DangFKKLOSS18}}]\label{lem:mutime}
  Let $n \in \N_{\geq 1}$, let $k \in [n/2]$, let $d \in [0 .. k]$, and let $p_m = \Theta(\frac{1}{n})$.
  Furthermore, let $x, y \in \{0, 1\}^n$ with $|x|_1 = |y|_1 = n - k$ and such that their Hamming distance is $2d$.
  Then the probability that the result of uniform crossover of~$x$ and~$y$, followed independently by standard bit mutation with mutation rate~$p_m$, is the all-$1$s string is $\Omega(4^{-d} n^{-k+d})$.
\end{lemma}

\begin{proof}
  Since~$x$ and~$y$ each have $n - k$ $1$s and a Hamming distance of~$2d$, they have $n - k - d$ $1$s in common.
  Analogously, they have $k - d$ $0$s in common.
  Hence, uniform crossover produces an offspring~$z$ with $n - k + d$ $1$s by choosing a~$1$ each time that~$x$ and~$y$ differ.
  The probability for this event is $2^{-2d} = 4^{-d}$.
  Afterward mutation creates the all-$1$s string out of~$z$ by flipping its remaining $k - d$ $0$s to~$1$s (and not flipping any other bits).
  Due to the assumption $p_m = \Theta(\frac{1}{n})$ and due to $(1 - p_m)^{n} = \Theta(1)$, the probability of this event is $(1 - p_m)^{n - k + d} p_m^{k - d} \geq (1 - p_m)^{n} p_m^{k - d} = \Omega(n^{-k + d})$.
  Since crossover and mutation act independently, the result follows.
\end{proof}

In the second phase, we show via the negative-drift theorem below that the diversity of the population on the plateau persist with high probability for a time exponential in~$\mu$.
This theorem is a version of~\cite[Theorem~$3$]{Kotzing16} where the drift does not need to be negative for the entire search space but only for parts of it.
\begin{theorem}[{\cite[Corollary~$3.24$]{Krejca19}}]
  \label{thm:drift}
  Let $(X_t)_{t \in \N}$ be a random process over~$\R$ adapted to a filtration $(\calF_t)_{\in \N}$.
  Further, let $X_0 \leq 0$, let $b \in \R_{> 0}$, and let $T=\inf\{t \in \N \mid X(t) \geq b \}$.
  Suppose that there are constants $a \in \R_{\leq 0}$, $c \in (0,b)$, and $\eps \in \R_{<0}$ such that, for all $t \in \N$, it holds that
  \begin{enumerate}
    \item  $\E[ (X_{t+1} -X_{t}) \cdot \mathbf{1}_{ X_t \geq a } \cdot \mathbf{1}_{ T <t }\mid \calF_{t} ]  \leq \varepsilon \cdot \mathbf{1}_{ X_t \geq a } \cdot \mathbf{1}_{ T <t }$, that

    \item $|X_{t+1} -X_{t} | \cdot \mathbf{1}_{ X_t \geq a } \cdot \mathbf{1}_{ T <t }  < c\cdot \mathbf{1}_{ X_t \geq a } + \mathbf{1}_{ X_t < a } $, and that

    \item $ X_{t+1} \cdot \mathbf{1}_{ X_t < a } \cdot \mathbf{1}_{ T <t } \leq 0 $.
  \end{enumerate}
  Then, for all $t \in \N$, it holds that $\Pr[T\leq t] \leq t^{2} \exp\left(- \frac{b|\varepsilon|}{2 c^2}\right)$.
\end{theorem}

\section{Runtime Analysis}\label{sec:results}

We analyze the \mga with mutation rate $\frac{\chi}{n}$, where $\chi \in \R_{>0}$ is a constant, and with constant crossover rate, following the ideas of Dang et~al.~\cite{DangFKKLOSS18}.
We show that once the entire population of the \mga is on the plateau, the population quickly reaches a state where the largest genotype class has at most $\frac{3}{4}\mu$ individuals for a time exponential in~$\mu$ with high probability.
During this time, applying crossover to any two individuals with Hamming distance at least~$2$, followed by mutation, results in a speed-up of order $\Omega(n)$ over only applying crossover.
\begin{theorem}\label{thm:runtime}
  Let $\chi \in \Theta(1)$.
  Consider the \mga with population size $\mu = O(n)$, crossover rate $p_c = \Theta(1)$, and mutation rate $p_m= \frac{\chi}{n}$ optimizing $\jumpk$ with $k= o(n)$.
  Last, let $C = \frac{1 + \frac{7}{4}\chi}{512e}p_c$.
  Then the expected runtime is
  \begin{align*}
    O\left(n\sqrt{k}(\mu \log(\mu)+ \log(n)) + \frac{\mu n + \mu^2 \log(\mu)}{n^{-k + 1} \min(\exp(\frac{C\mu}{2}), n^{k-1})} + n^{k - 1}\right).
  \end{align*}
\end{theorem}

In our analysis, we assume that the entire population of the \mga is on the plateau, as Theorem~\ref{thm:timeplateau} covers the expected time it takes until the algorithm reaches this state (or finds the optimum).
Like Dang et~al.~\cite{DangFKKLOSS18}, we consider the population dynamics on the plateau and show that sufficiently diverse individuals show up and remain for a long time.

To make things more formal, we refer to identical individuals as a \emph{species}.
For a given species, we call an individual belonging to this species a \emph{\yIndividual}, otherwise we say it is a \emph{\nonyIndividual}.
Since the population is on the plateau, individuals from different species have a Hamming distance of at least~$2$.
Hence, for a constant $c \in (0, 1)$, if the largest species has a size of at most~$c\mu$ and due to the constant crossover rate, there is a constant probability to select two individuals for crossover that have a Hamming distance of at least~$2$.
By Lemma~\ref{lem:mutime}, the probability to create the optimum via such a crossover followed by mutation is at least $\Omega(n^{-k + 1})$,
which concludes the analysis.

In order to make these points rigorous, we carefully analyze how the size of a species evolves over the iterations.
To this end, for species~$s$ (where~$s$ is just a symbol), we view its size~$Y_s$ as a random process over $[0 .. \mu]$.
In the following, we analyze the transition probabilities of this random process, which we define for all $y \in [\mu]$, all $t \in \N$, and all events~$A$ with non-zero probability as
\begin{align}
  \label{eq:transitionProbabilities}
  p^{(s, t)}_{+}(y \mid A)&\coloneqq \Pr[Y_s(t+1)-Y_s(t)= 1 \mid Y_s(t) = y, A ] \text{ and}\\
  \notag
  p^{(s, t)}_{-}(y \mid A)&\coloneqq \Pr[Y_s(t+1)-Y_s(t)= -1 \mid Y_s(t) = y, A ].
\end{align}

When analyzing the transition probabilities stated above, we make a case distinction with respect to how the offspring in an iteration is created.
In total, we consider three different cases:
Lemmas~\ref{lem:rdrift} and~\ref{lem:ldrift} assume that crossover (followed by mutation) is performed with parents whose Hamming distance is at most~$2$.
Lemma~\ref{lem:hamdrift4} assumes that crossover is performed with parents whose Hamming distance is larger than~$2$.
And Lemma~\ref{lem:mutationdrift} assumes that only mutation is performed.

Our analysis differs from the one of Dang et~al.~\cite{DangFKKLOSS18} in how carefully we analyze the transition probabilities above, and also that we consider (constant) crossover rates less than~$1$.
Dang et~al. show that the largest species remains small for an amount of time that is equal to that of a fair random walk (about~$\mu^2$ iterations).
In contrast, we show that there is actually a drift toward decreasing the species, which grows as the size of the species decreases to around~$\frac{\mu}{2}$.
To this end, we use the results by Dang et~al. that show how likely the size of a species is to increase.
However, we go into detail about how likely the size decreases, as this is important for the drift, and we also consider the case of only applying mutation in a single iteration.

The following lemma bounds the probabilities of a species increasing as well as the probability of a monomorphic population to create a new genotype.
We phrase the lemma in a way for which it actually holds, which is more permissive than what Dang et~al.~\cite{DangFKKLOSS18} claim, when looking into their proof.
\begin{lemma}[{\cite[Lemma~$3$]{DangFKKLOSS18}}]
  \label{lem:rdrift}
  Let $\chi \in \Theta(1)$.
  Consider a single iteration $t \in \N$ of the \mga with population size $\mu \in \N_{\geq 2}$, crossover rate $p_c = \Theta(1)$, and mutation rate $p_m= \frac{\chi}{n}$ optimizing $\jumpk$ with $k= o(n)$.
  Let~$A$ denote the event that the entire population~$P^{(t)}$ is on the plateau and that the offspring produced this iteration is created via crossover of two parents whose Hamming distance is at most~$2$ (followed by mutation).
  Assume that~$A$ has a positive probability of occurring.
  Last, let~$s$ be a species of~$P^{(t)}$.
  Then, for the transition probabilities defined in equation~\eqref{eq:transitionProbabilities} and for all $y \in [\mu]$, it holds that
  \begin{align*}
    p^{(s, t)}_{+}(y \mid A)
    &\leq  \frac{(\mu-y)y(\mu+y)}{2(\mu+1)\mu^2}\left( 1-\frac{\chi}{n}\right)^{n} + O\left(\left(\frac{\mu-y}{\mu}\right)^2 \frac{1}{n}\right)  \text{, and} \\
    p^{(s, t)}_{-} (\mu \mid A)& = \Omega \left(\frac{k}{n}\right).
  \end{align*}
\end{lemma}

In the following lemma, in the same setting of Lemma~\ref{lem:rdrift}, we bound the probability of the size of a species decreasing for the sizes not covered by Lemma~\ref{lem:rdrift}.
It follows closely the proof of \cite[Lemma~$3$]{DangFKKLOSS18}, but it is more detailed about the case when two \nonyIndividuals are selected as parents for the crossover.
\begin{lemma}\label{lem:ldrift}
  Let $\chi \in \Theta(1)$.
  Consider a single iteration $t \in \N$ of the \mga with population size $\mu \in \N_{\geq 2}$, crossover rate $p_c = \Theta(1)$, and mutation rate $p_m= \frac{\chi}{n}$ optimizing $\jumpk$ with $k= o(n)$.
  Let~$A$ denote the event that the entire population~$P^{(t)}$ is on the plateau and that the offspring produced this iteration is created via crossover of two parents whose Hamming distance is at most~$2$ (followed by mutation).
  Assume that~$A$ has a positive probability of occurring.
  Last, let~$s$ be a species of~$P^{(t)}$.
  Then, for the transition probabilities defined in equation~\eqref{eq:transitionProbabilities} and for all $y \in [\mu - 1]$, it holds that
  \[
    p^{(s, t)}_{-} (y \mid A) \geq \frac{y(\mu-y)\left( \mu\left(1 + \frac{1}{2}\chi\right) + \frac{1}{2}y\chi \right) } {2(\mu +1) \mu^2} \left( 1-\frac{\chi}{n}\right)^{n}.
  \]
\end{lemma}

\begin{proof}
  Decreasing the size of~$s$ entails that an offspring that is a \nonyIndividual is created (as clarified in the lemma) and then not removed afterward during the selection.
  Let~$q(y \mid A)$ denote the probability (conditional on~$A$) that a \nonyIndividual is created by mutation and crossover.
  Assuming that the offspring is on the plateau, then, since all individuals are on the plateau, the individual for deletion is selected uniformly at random.
  Thus, since we consider a \yIndividual to be removed, we get
  \begin{align}
    \label{eq:ldrift:qSubstitution}
    p^{(s, t)}_{-}(y \mid A) \geq q(y \mid A) \left(\frac{y}{\mu+1}\right).
  \end{align}

  We bound~$q(y \mid A)$ via a case distinction with respect to what species the parents of the crossover belong to.

  \textbf{Case~1.}
  The parents are a \yIndividual and a \nonyIndividual.
  The probability to select such parents is $2 \frac{y}{\mu} \frac{\mu-y}{\mu}$.
  We now make another case distinction with respect to the outcome of the crossover.

  \textbf{Case~1.1.}
  The crossover result if a copy of the \nonyIndividual parent.
  This occurs with probability $\frac{1}{4}$.
  Afterward, the crossover offspring is not changed by mutation with probability $(1-\frac{\chi}{n})^{n}$.

  \textbf{Case~1.2.}
  Crossover produces an individual with $k - 1$ $0$s.
  This occurs with probability $\frac{1}{4}$.
  Afterward, mutation creates a \nonyIndividual by flipping a single of the $n-k$ $1$s that do not lead to a \yIndividual.
  The probability of this mutation is $(n - k)\frac{\chi}{n}(1-\frac{\chi}{n})^{n - 1}$.

  \textbf{Case~1.3.}
  Crossover produces an individual with $k + 1$ $0$s.
  This occurs with probability $\frac{1}{4}$.
  Afterward, mutation creates a \nonyIndividual by flipping a single of $k$ $1$s that do not lead back to a \yIndividual.
  The probability of this mutation is $k\frac{\chi}{n}(1-\frac{\chi}{n})^{n - 1}$

  \textbf{Concluding Case~1.}
  Combining the three cases yields
  \begin{align*}
    &\frac{2y}{\mu}\frac{\mu - y}{\mu}\left(\frac{1}{4} \left( 1-\frac{\chi}{n} \right)^n +\frac{1}{4} (n-k) \frac{\chi}{n} \left( 1-\frac{\chi}{n}\right)^{n-1} \right.\\
    &\qquad\qquad\quad \left.+ \frac{1}{4}k \frac{\chi}{n} \left(1-\frac{\chi}{n}\right)^{n-1}\right)\\
    &= \frac{1}{2} \frac{y}{\mu}\frac{\mu - y}{\mu} \left( 1-\frac{\chi}{n}\right)^{n}  \left(1 + \chi \left(1-\frac{\chi}{n}\right)^{-1}\right)\\
    &\geq \frac{1}{2} \frac{y}{\mu}\frac{\mu - y}{\mu} \left( 1-\frac{\chi}{n}\right)^{n}  \left(1 + \chi\right).
  \end{align*}

  \textbf{Case~2.}
  The parents are two \nonyIndividuals.
  The probability to select such parents is $(\frac{\mu-y}{\mu})^2$.
  We now make a case distinction with respect to the Hamming distance of the parents, which is either~$0$ or~$2$, since we condition on~$A$.

  \textbf{Case~2.1.}
  The parents have Hamming distance~$0$.
  Crossover produces a copy of the parents.
  If the following mutation does not flip any bits, which occurs with probability~$(1-\frac{\chi}{n})^{n}$, the result is on the plateau.

  \textbf{Case~2.2.}
  The parents have Hamming distance~$2$.
  Then, we are in a similar scenario as in Case~$1$ above.
  The difference to Case~$1.1$ is that the probability of the crossover producing a \nonyIndividual is~$\frac{1}{2}$, as both parents are \nonyIndividuals in our current case, so copying them is sufficient.

  The difference to Case~$1.2$ is that there are \emph{at least} $n - k$ $1$s that, if flipped, do not create a \yIndividual, as we do not know how many~$1$s the offspring shares with a \yIndividual.
  We know though that there is only a single genotype for the \yIndividuals, so there is only at most one possibility for the crossover offspring to hit this genotype if a single bit is flipped.

  The difference to Case~$1.3$ is the same as the one to Case~$1.2$ but when swapping~$1$s and~$0$s in the argument.

  \textbf{Concluding Case~2.}
  Since Case~$2.2$ is a lower bound of Case~$2.1$, we only use it as a bound for Case~2.
  Using the respective probabilities, borrowing from Case~1 as explained, we get that the probability of Case~2 is bounded from below by
  \begin{align*}
    &\left(\frac{\mu-y}{\mu}\right)^2 \left(\frac{1}{2} \left( 1-\frac{\chi}{n} \right)^n +\frac{1}{4} (n-k) \frac{\chi}{n} \left( 1-\frac{\chi}{n}\right)^{n-1}\right.\\
    &\qquad\qquad\quad \left.+ \frac{1}{4}k \frac{\chi}{n} \left(1-\frac{\chi}{n}\right)^{n-1}\right)\\
    &= \frac{1}{2} \left(\frac{\mu-y}{\mu}\right)^2 \left(1-\frac{\chi}{n}\right)^{n} \left(1 + \frac{1}{2}\chi \left(1-\frac{\chi}{n}\right)^{-1}\right)\\
    &\geq \frac{1}{2} \left(\frac{\mu-y}{\mu}\right)^2 \left(1-\frac{\chi}{n}\right)^{n} \left(1 + \frac{1}{2}\chi\right).
  \end{align*}

  \textbf{Combining the cases.}
  Overall, the probability of adding a \nonyIndividual to the population before selection is
  \begin{align*}
    q(y)&\geq \frac{1}{2} \frac{y}{\mu} \frac{\mu-y}{\mu} \left( 1-\frac{\chi}{n}\right)^{n} \left( 1 +\chi \right)  + \frac{1}{2} \left(\frac{\mu-y}{\mu}\right)^2 \left( 1-\frac{\chi}{n}\right)^{n} \left(1 + \frac{1}{2}\chi\right) \\
    &= \frac{1}{2} \frac{1}{\mu} \frac{\mu-y}{\mu} \left( y(1 + \chi) + (\mu - y)\left(1 + \frac{1}{2}\chi\right)\right) \left( 1-\frac{\chi}{n}\right)^{n}\\
    &= \frac{1}{2} \frac{1}{\mu} \frac{\mu-y}{\mu} \left( \mu\left(1 + \frac{1}{2}\chi\right) + \frac{1}{2}y\chi \right) \left( 1-\frac{\chi}{n}\right)^{n} .
  \end{align*}
  Substituting $q(y \mid A)$ into equation~\eqref{eq:ldrift:qSubstitution} concludes the proof.
\end{proof}

The following lemma, by Dang et~al.~\cite{DangFKKLOSS18}, considers the case where crossover is applied to two parents with Hamming distance at least~$4$, as it was excluded from Lemmas~\ref{lem:rdrift} and~\ref{lem:ldrift}.
In contrast to Lemmas~\ref{lem:rdrift} and~\ref{lem:ldrift}, this lemma requires that the considered species has a size of at least~$\frac{\mu}{2}$.
\begin{lemma}[{\cite[Lemma~4]{DangFKKLOSS18}}]\label{lem:hamdrift4}
  Let $\chi \in \Theta(1)$.
  Consider a single iteration $t \in \N$ of the \mga with population size $\mu \in \N_{\geq 2}$, crossover rate $p_c = \Theta(1)$, and mutation rate $p_m= \frac{\chi}{n}$ optimizing $\jumpk$ with $k= o(n)$.
  Let~$A'$ denote the event that the entire population~$P^{(t)}$ is on the plateau and that the offspring produced this iteration is created via crossover (followed by mutation) of two parents whose Hamming distance is larger than~$2$, i.e., at least~$4$.
  Assume that~$A'$ has a positive probability of occurring.
  Last, let~$s$ be a species of~$P^{(t)}$.
  Then, for the transition probabilities defined in equation~\eqref{eq:transitionProbabilities} and for all $y \in [\frac{\mu}{2} .. \mu]$, it holds that $p^{(s, t)}_{-}(y \mid A') \geq 2p^{(s, t)}_{+}(y \mid A')$.
\end{lemma}

The following lemma considers the transition probabilities of a species under the assumption that only mutation is applied.
\begin{lemma}\label{lem:mutationdrift}
  Let $\chi \in \Theta(1)$.
  Consider a single iteration $t \in \N$ of the \mga with population size $\mu \in \N_{\geq 2}$, crossover rate $p_c = \Theta(1)$, and mutation rate $p_m= \frac{\chi}{n}$ optimizing $\jumpk$ with $k= o(n)$.
  Let~$B$ denote the event that the entire population~$P^{(t)}$ is on the plateau and that the offspring produced this iteration is created via mutation (and no crossover).
  Assume that~$B$ has a positive probability of occurring.
  Last, let~$s$ be a species of~$P^{(t)}$.
  Then, for the transition probabilities defined in equation~\eqref{eq:transitionProbabilities} and for all $y \in [\mu]$, it holds that
  \begin{align*}
    p^{(s, t)}_{+}(y \mid A)& = \frac{y(\mu-y)}{\mu(\mu+1)}\left(1-\frac{\chi}{n}\right)^{n}+O\left(\frac{(\mu-y)^{2}}{n\mu^{2}}\right) \text{, and}\\
    p^{(s, t)}_{-}(y \mid A)& \geq \frac{y(\mu-y)}{\mu(\mu+1)}\left(1-\frac{\chi}{n}\right)^{n}.
  \end{align*}
\end{lemma}

\begin{proof}
  \textbf{Increase.}
  We start by bounding $p^{(s, t)}_{+}(y)$ from above.
  In order for the number of \yIndividuals to increase, mutation needs to produce a \yIndividual, which happens in the following two ways:
  \begin{enumerate}
      \item A \yIndividual is selected and mutation does not flip any bits.
        This occurs with probability $(1-\frac{\chi}{n})^{n}$.
      \item A \nonyIndividual is selected.
        Hence, at least one specific position must be flipped, occuring with probability $O(\frac{1}{n})$.
  \end{enumerate}
  After mutation, a \nonyIndividual needs to be removed from the population in order for the number of \yIndividuals to increase.
  The probability to remove a \nonyIndividual is $\frac{\mu-y}{\mu+1}$.
  Hence, the overall probability for an increase is at most $\frac{y(\mu-y)}{\mu(\mu+1)}(1-\frac{\chi}{n})^{n}+O\left(\frac{(\mu-y)^{2}}{n\mu^{2}}\right)$.

  \textbf{Decrease.}
  We now bound~$p^{(s, t)}_{-}(y)$ from below.
  To this end, we only consider the case that (i) a \nonyIndividual is selected for mutation, (ii) that it is copied via mutation (i.e., flipping no bits), and (iii) that a \yIndividual is removed during selection.
  Event (i) occurs with probability $\frac{\mu - y}{\mu}$, event (ii) with probability $(1-\frac{\chi}{n})^{n}$, and event (iii) with probability $\frac{y}{\mu+1}$.
  Hence, the overall probability for a decrease is at least $\frac{y(\mu-y)}{\mu(\mu+1)}(1-\frac{\chi}{n})^{n}.$
\end{proof}

The next lemma bounds the expected time for a species to reach a size of at most~$\frac{\mu}{2}$, by Dang et~al.~\cite{DangFKKLOSS18}.
The authors only consider the largest species, but since the proof relies on the transition probabilities we stated so far in this section, which are not specific to the largest species, the lemma is actually more general.
\begin{lemma}[{\cite[Lemma~$6$]{DangFKKLOSS18}}]\label{reducesize}
  Let $\chi \in \Theta(1)$.
  Consider the \mga with population size $\mu = O(n)$, crossover rate $p_c = \Theta(1)$, and mutation rate $p_m= \frac{\chi}{n}$ optimizing $\jumpk$ with $k= o(n)$.
  Further, let~$s$ be a species, and let $t^* \in \N$ be an iteration such~$P^{(t^*)}$ is entirely on the plateau.
  In addition, let~$Y_s(t^*)$ be defined as above equation~\eqref{eq:transitionProbabilities}.
  Let $T = \inf\{t \in \N \mid Y_s(t^* + t) \leq \frac{\mu}{2}\}$.
  Then $\E[T \mid Y_s(t^*)] = O(\mu n + \mu^2 \log(\mu))$.
\end{lemma}

The following lemma shows that once a species has a size of at most~$\frac{\mu}{2}$, then, for any constant $\lambda \in (\frac{1}{2}, 1)$, its size does not reach $\lambda\mu$ with high probability in a number of iterations exponential in~$\mu$.
\begin{lemma}\label{lem:probound}
  Let $\chi \in \Theta(1)$.
  Consider the \mga with population size $\mu \in \N_{\geq 2}$, crossover rate $p_c = \Theta(1)$, and mutation rate $p_m= \frac{\chi}{n}$ optimizing $\jumpk$ with $k= o(n)$.
  Further, let~$s$ be a species, and let $t^* \in \N$ be an iteration such~$P^{(t^*)}$ is entirely on the plateau.
  Assume that~$Y_s(t^*)$ (as defined above equation~\eqref{eq:transitionProbabilities}) is at least~$\frac{\mu}{2}$.
  Let $\lambda \in (\frac{1}{2}, 1)$, and let $T = \inf\{t \in \N \mid Y_s(t + t^*) \geq \lambda\mu\}$.
  Last, let $C = \frac{(2\lambda - 1)(1 + (1 + \lambda)\chi)}{256e}p_c$.
  Then, for all $t \in \N$, it holds that $\Pr[T\leq t] \leq t^{2} \cdot \exp(-C\mu)$.
\end{lemma}

\begin{proof}
  We aim to apply the negative-drift theorem (Theorem~\ref{thm:drift}) with $(\calF_t)_{t \in \N}$ being the canonical filtration of $(Y_s(t + t^*))_{t \in \N}$.
  We consider the process $(X_t)_{t \in \N}$, which is defined such that, for all $t \in \N$, it holds that $X_t = Y_s(t + t^*)-\frac{\mu}{2}$.
  Furthermore, we choose the variables of Theorem~\ref{thm:drift} by defining that $b=(\lambda-\frac{1}{2})\mu$, $c=1$, $\eps = -\frac{1 + (1+\lambda)\chi}{64e}p_c$ being a constant, and $a=0$.

  By the definition of~$X$, it holds that $T = \inf\{t \in \N \mid X_t \geq (\lambda - \frac{1}{2})\mu = b\mu\}$.

  We make sure that the different conditions of Theorem~\ref{thm:drift} are all met.
  For the sake of readability, we omit the indicator functions in this proof.

  \textbf{Condition~(i).}
  Let $t \in \N$, and let $y = Y_s(t + t^*)$.
  By the definition of~$Y_s$ and of~$X$, it that
  \begin{equation}\label{main:2}
    \E[X_{t+1}-X_t  \mid {\calF_t} ] = p^{(s, t + t^*)}_{+}(y \mid {\calF_t})-p^{(s, t + t^*)}_{-}(y \mid {\calF_t}).
  \end{equation}

  We aim to bound equation~\eqref{main:2} from above by using the bounds of Lemmas~\ref{lem:rdrift} and~\ref{lem:ldrift}.
  To this end, we define the constants
  \begin{align*}
      C_{+}&= \frac{(\mu-y)y(\mu+y)}{2(\mu+1)\mu^2} \textrm{ and}\\
      C_{-}&= \frac{y(\mu-y)\left( \mu\left(1 + \frac{1}{2}\chi\right) + \frac{1}{2}y\chi \right) } {2(\mu +1) \mu^2}.
  \end{align*}

  We now make a case distinction with respect to the different bounds on the transition probabilities from the different settings of Lemmas~\ref{lem:rdrift}, \ref{lem:ldrift}, \ref{lem:hamdrift4}, and~\ref{lem:mutationdrift}.

  \textbf{Case~1.}
  We consider the setting of Lemmas~\ref{lem:rdrift} and~\ref{lem:ldrift}.
  Let~$A$ denote the respective event from those lemmas.
  Since there is always the chance to pick the same individual twice for crossover, $\Pr[A] > 0$.
  The respective bounds from Lemmas~\ref{lem:rdrift} and~\ref{lem:ldrift} yield that, noting that $y < \mu$, as we assume that $y \in [\frac{\mu}{2}, \lambda\mu]$ but omit the indicator functions,
  \begin{align*}
    &p^{(s, t + t^*)}_{+}(y \mid A)-p^{(s, t + t^*)}_{-}(y \mid A)\\
    &\leq \left( C_{+}\left(1-\frac{\chi}{n}\right)^n-C_{-}\left(1-\frac{\chi}{n}\right)^n + O\left(\left(\frac{\mu-y}{\mu}\right)^2 \frac{1}{n}\right) \right)   \\
    & \leq \left( -C_{+}\left(1-\frac{\chi}{n}\right)^n\frac{\mu\left(1 + \frac{1}{2}\chi\right) + \frac{1}{2}y\chi}{\mu+y}+o(1) \right) \\
    &\leq \left(-\frac{(\mu-y)y\left(\mu\left(1 + \frac{1}{2}\chi\right) + \frac{1}{2}y\chi\right)}{2(\mu+1)\mu^{2}}\left( 1-\frac{\chi}{n} \right)^n + o(1)\right).
  \end{align*}
  Using that $y \in [\frac{\mu}{2}, \lambda\mu]$ and that $\left(1-\frac{\chi}{n}\right)^n \geq \frac{1}{e}\left(1-\frac{\chi}{n}\right)$, we get
  \begin{equation}\label{main:3}
    \E[X_{t+1}-X_t \mid \calF_t; A] \leq -\frac{2 + (1+\lambda)\chi}{16e} + o(1).
  \end{equation}

  \textbf{Case~2.}
  We consider the setting of Lemma~\ref{lem:hamdrift4}.
  Let~$A'$ denote the respective event from this lemma.
  Assume that $\Pr[A'] > 0$.
  By Lemma~\ref{lem:hamdrift4}, we get
  \begin{align}
    \label{main:4}
    p^{(s, t + t^*)}_{+}(y \mid A')-p^{(s, t + t^*)}_{-}(y \mid A') \leq -p^{(s, t + t^*)}_{+}(y \mid A') \leq 0.
  \end{align}

  \textbf{Case~3.}
  We consider the setting of Lemma~\ref{lem:mutationdrift}.
  Let~$B$ denote the respective event from this lemma.
  Assume that $\Pr[B] > 0$.
  By Lemma~\ref{lem:mutationdrift}, we get
  \begin{equation}\label{main:7}
    p^{(s, t + t^*)}_{+}(y \mid B)-p^{(s, t + t^*)}_{-}(y \mid B) \leq O\left(\frac{(\mu-y)^{2}}{n\mu^{2}}\right) = O\left(\frac{1}{n}\right).
  \end{equation}

  \textbf{Combining the cases.}
  Since the events~$A$, $A'$, and~$B$ partition the possible events of how the population changes in a single iteration, we determine $p^{(s, t + t^*)}_{+}(y)-p^{(s, t + t^*)}_{-}(y)$ via the law of total probability.
  To this end, we first discuss whether the events~$A$, $A'$, and~$B$ have positive probability for the relevant setting of $y \in [\frac{\mu}{2}, \lambda\mu]$.

  Since $p_c > 0$, crossover has a positive probability of being used in each iteration.
  Conditional on crossover being used, since $y \geq \frac{\mu}{2}$, if crossover picks a \yIndividual twice, occurring with probability $(\frac{y}{\mu})^2 \geq \frac{1}{4}$, then this is sufficient for~$A$ to occur.
  Hence, $\Pr[A] \geq \frac{1}{4}$.

  The probabilities of~$A'$ and~$B$ may be~$0$.
  If $\Pr[B] = 0$, then $\Pr[A \cup A'] = 1$.
  And if $\Pr[A'] = 0$, then $\Pr[A \mid \overline{B}] = 1$.
  Since, by Case~2, the drift in the event of~$A'$ (equation~\eqref{main:4}) is negative and since event~$A$ covers the case that crossover is applied entirely otherwise, we ignore~$A'$ in the following.
  Further, since, by Case~3, the drift in the event of~$B$ (equation~\eqref{main:7}) is positive, we pessimistically estimate its probability by~$1$.

  Together, these considerations yield, combining equations~\eqref{main:3}, \eqref{main:4}, and~\eqref{main:7}, as well as equation~\eqref{main:2} that
  \begin{align*}
    &\E[X_{t+1}-X_t  \mid {\calF_t} ] \leq \left(-\frac{2 + (1+\lambda)\chi}{16e} + o(1)\right) \Pr[A \mid \calF_t] + O\left(\frac{1}{n}\right)\\
    &\leq -\frac{1 + (1+\lambda)\chi}{16e} p_c \cdot \frac{1}{4}
    = -\frac{1 + (1+\lambda)\chi}{64e} p_c = \eps.
  \end{align*}

  \textbf{Condition~(ii).}
  This condition is satisfied for $c = 1$, as~$y$ changes by at most~$1$ per iteration.

  \textbf{Condition~(iii).}
  If $X_t < 0 = a$, then, by the definition of~$X$, it follows that $y < \frac{\mu}{2}$.
  Since~$y$ changes by at most~$1$ in a single iteration, it follows that $Y_s(t + 1 + t^*) \leq \frac{\mu}{2}$, which is equivalent to $X_{t + 1} \leq 0$.

  \textbf{Applying the drift theorem.}
  Since all conditions of Theorem~\ref{thm:drift} are met, we apply it and get for all $t \in \N$ that
  \[
    \Pr[T\leq t] \leq t^{2} \exp\left(-\frac{(2\lambda-1)|\eps|\mu}{4}\right).
  \]
  Noting that $C = \frac{(2\lambda - 1)(1 + (1 + \lambda)\chi)}{256e}p_c$ concludes the proof.
\end{proof}

We have all the tools we require to prove our main result.
\begin{proof}[Proof of Theorem~\ref{thm:runtime}]
  We split the runtime into two parts, mainly focusing on the second part.
  The first part considers the expected number of fitness function evaluations until the entire population of the \mga is on the plateau or until the optimum is found.
  The second part assumes that the entire population of the \mga is on the plateau, and it considers the expected number of fitness function evaluations (which is equivalent to iterations here, since we evaluate a single new individual each iteration) until the optimum of $\jumpk$ is found.
  By the linearity of expectation, the overall expected runtime of the \mga is at most the sum of the expected times of these two parts, since the second part ignores whether the optimum is found in the first part.

  \textbf{First part.}
  Theorem~\ref{thm:timeplateau} bounds the expected time for the first part by $O(n\sqrt{k}(\mu \log(\mu)+ \log(n)))$.
  Hence, for the remainder ot this proof, we focus on the second part.

  \textbf{Second part.}
  Since we count iterations from the first point in time that the entire population is on the plateau, we assume that this time point is~$0$, simplifying the forthcoming calculations.
  Let~$T$ be the first point in time (starting from the newly defined~$0$) such that the \mga creates the optimum.

  We split~$T$ into \emph{phases} of deterministic length~$\ell$ (which we specify later).
  If the \mga does not find the optimum during a single phase, we go into the next phase, repeating our arguments.
  To this end, let~$R$ denote the number of phases that pass until the \mga creates the optimum for the first time.
  By definition, it then holds that $T \leq R\ell$.
  Consequently, $\E[T] \leq \E[R] \ell$.

  Since the entire population is on the plateau until the optimum is found for the first time, each phase (except for the last one) has the same, independent probability of finding the optimum.
  Let~$q$ be this probability, noting that $q > 0$.
  Since the probability of finding the optimum in the last phase is~$1$, the probability~$q$ is a lower bound for this phase as well.
  Thus,~$R$ is stochastically dominated by a geometric random variable with success probability~$q$, and it follows that $E[R] \leq \frac{1}{q}$.
  Together with the arguments from the previous paragraph, this yields overall $E[T] \leq \frac{\ell}{q}$.

  \textbf{Bounding the success probability.}
  For the remaining part, we aim to bound~$q$.
  Let $c_1 \in \N$ be a sufficiently large constant from the big-O result in Lemma~\ref{reducesize}.
  Further, let $c_2 \in \N$ be a sufficiently small constant from the big-O result in Lemma~\ref{lem:mutime}, for the case $d = 2$.
  In addition, let $t = \min(\exp(\frac{C\mu - 1}{2}),\frac{1}{c_2}n^{k-1})$.
  Last, let $\ell = 2c_1(\mu n + \mu^2 \log(\mu)) + t$.
  We consider a single phase.

  \textbf{Achieving a sufficient diversity.}
  During a phase, we wait for the first iteration~$L$ (relative to the start of the phase) until the largest species has a size~$y$ of at most~$\frac{\mu}{2}$.
  By Lemma~\ref{reducesize} and by definition of~$c_1$, the \emph{expected} time for this is at most $c_1(\mu n + \mu^2 \log(\mu))$.
  Since~$L$ is a non-negative random variable, it follows by Markov's inequality that $\Pr[L \geq 2\E[L]] \leq \frac{1}{2}$.
  We assume in the following that $L < 2\E[L] = 2c_1(\mu n + \mu^2 \log(\mu))$.
  Note that, it holds that $2c_1(\mu n + \mu^2 \log(\mu)) \leq \ell$.

  \textbf{Lasting diversity.}
  Given that $y \leq \frac{\mu}{2}$ during the phase, by Lemma~\ref{lem:probound}, choosing $\lambda = \frac{3}{4}$, it follows that $y < \lambda\mu$ for~$t$ iterations (as defined above) with probability at least $1 - \frac{1}{e}$.
  We assume that this event occurs.
  Note that by the definition of~$\ell$, a phase is long enough to account for this length.

  \textbf{Success probability in a single iteration.}
  Since $y < \lambda\mu$, there are at least two different species in the population.
  Since the population is on the plateau, different species have a Hamming distance of at least~$2$.
  Thus, during each of the at least~$t$ iterations, the probability that the \mga performs crossover on two individuals of Hamming distance at least~$2$ is given by the probability that the second parent belongs to a different species than the first.
  Since $y < \lambda\mu$, this probability is at least $\frac{1 - \lambda}{\mu} = \frac{1}{4}$.
  Conditional on choosing such parents, Lemma~\ref{lem:mutime} states that the optimum is produced with a probability of $\Omega(n^{-k + 1})$.
  Overall, the probability that, in a single iteration, the optimum is created is at least $p_c \cdot \frac{1}{4} \cdot \Omega(n^{-k + 1}) = \Omega(n^{-k + 1})$.
  Let~$q'$ denote this probability.

  \textbf{Success probability during lasting diversity.}
  By the arguments discussed in the previous two paragraphs, the \mga has~$t$ iterations during which the success probability of each is at least~$q'$.
  The probability that the optimum is created during any of these iterations is at least $1 - (1 - q')^t \geq 1 - \exp(-q't)$.
  By the definition of~$t$, it holds that $q't \leq 1$.
  Thus, since it holds for all $x \in [0, 1]$ that $\exp(-x) \leq 1 - \frac{x}{2}$, it follows that $1 - \exp(-q't) \geq \frac{q't}{2}$.

  \textbf{Concluding a phase.}
  Assuming that the different events discussed above occur during a phase, recalling the definition of~$q$, it follows that $q \geq \frac{1}{2} \cdot (1 - \frac{1}{e}) \cdot \frac{q't}{2} = \Omega(q't)$.

  \textbf{Concluding the second part.}
  With our bound for~$q$, we conclude $\E[T] \leq \frac{\ell}{q} = O\bigl(\frac{\ell}{q't}\bigr) = O\bigl(\frac{\mu n + \mu^2 \log(\mu)}{q't} + \frac{1}{q'}\bigr)$.
\end{proof}

\begin{figure*}
  \centering
  \begin{subfigure}{0.49 \textwidth}
    \includegraphics[width = \textwidth]{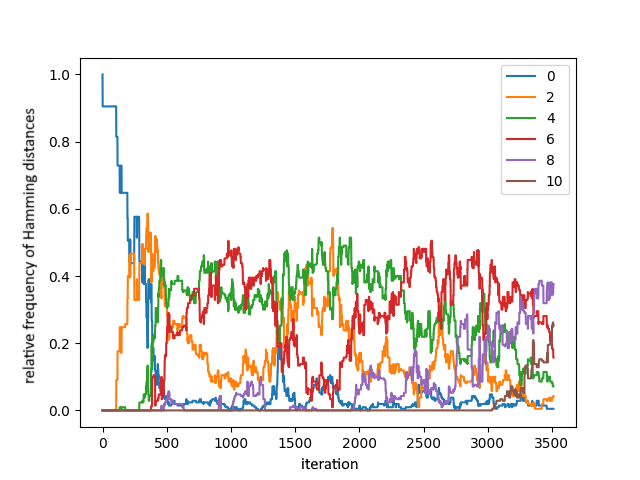}
    \caption{A single run of the \mga on $\jumpk$ with the parameters $n = 100$, $k = 5$, $\mu = 20$, $p_c = 1$, and $p_m = \frac{1}{n}$.}
    \label{fig:pairwiseHammingDistances:n100}
  \end{subfigure}
  \hfil
  \begin{subfigure}{0.49 \textwidth}
    \includegraphics[width = \textwidth]{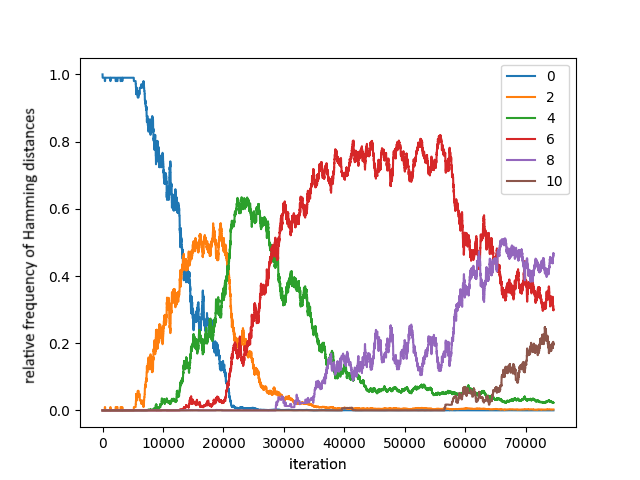}
    \caption{A single run of the \mga on $\jumpk$ with the parameters $n = 1000$, $k = 5$, $\mu = 200$, $p_c = 1$, and $p_m = \frac{1}{n}$.}
    \label{fig:pairwiseHammingDistances:n1000}
  \end{subfigure}
  \caption{The relative frequencies of the different Hamming distances of the population of the \mga (Algorithm~\ref{alg:gea}) on the plateau of $\jumpk$ (equation~\eqref{eq:jumpDefinition}), for the respective parameter settings.
    Each plot depicts a single run of the \mga on $\jumpk$ starting with the entire population on the plateau.
    The initial population is chosen such that it consists of a single species (i.e.,~$\mu$ copies of the same individual), chosen uniformly at random among all individuals with exactly~$k$ $0$s and $n - k$ $1$s.
    The run is stopped once the algorithm creates the global optimum.
    The $x$-axis denotes the number of iterations of the algorithm.
    The different colors refer to the different Hamming distances in the population.
    The $y$-axis depicts the relative frequency of each such Hamming distance among all $\binom{\mu}{2}$ pairs of individuals.
    Since all individuals are on the plateau, the Hamming distances are always even numbers.
    The maximum number is twice the gap size, i.e.,~$2k$.
    Please refer to Section~\ref{sec:experimentsDiversity} for more information.
  }
  \label{fig:pairwiseHammingDistances}
\end{figure*}

\section{Experimental Evaluation}
\label{sec:experiments}

Our runtime analysis in Section~\ref{sec:results}, especially Lemma~\ref{lem:probound}, shows that once the entire population of the \mga is on the plateau, there are phases of length exponential in~$\mu$ during which the largest species (that is, identical individuals) is at most~$\frac{3}{4}\mu$.
During this time, it is possible that various different species get produced.
The larger the Hamming distance of two species, the easier it is to find the optimum of $\jumpk$ (Lemma~\ref{lem:mutime}), potentially improving the runtime.
In fact, experiments by Dang et~al.~\cite[Fig.~$6$]{DangFKKLOSS18} suggest that the expected runtime of the \mga on $\jumpk$ is rather in the order of $n\ln(n) + 4^k$ for certain mutation rates.

Since the key to improving our result (Theorem~\ref{thm:runtime}) is to better understand how diverse the population on the plateau really is during phases of lasting diversity, we investigate these population dynamics empirically in this section.
As Dang et~al.~\cite{DangFKKLOSS18} present an extensive empirical analysis of the runtime of different algorithm variants on $\jumpk$, we focus on a quantity they did not cover, namely, the relative frequencies of all pairwise Hamming distances in the population.
A higher ratio of a larger Hamming distance implies a larger probability to create the optimum.
Thus, these numbers provide strong insights into how well the population is spread out in order to create the optimum.

\subsection{Frequencies of Pairwise Hamming Distances}
\label{sec:experimentsDiversity}
We analyze and discuss how the relative frequencies of pairwise Hamming distances in the population of the \mga on the plateau of $\jumpk$ evolve over the algorithm's iterations.
Figure~\ref{fig:pairwiseHammingDistances} shows our results and explains our experimental setup.
We see both in Figures~\ref{fig:pairwiseHammingDistances:n100} and~\ref{fig:pairwiseHammingDistances:n1000} that the different possible Hamming distances (from~$0$ to~$10$, in steps of~$2$) emerge in increasing order of thir value (that is, first distance~$0$, then~$2$, and so on).
This trend is qualitatively the same for both figures, with larger fluctuations for the smaller dimension size (Figures~\ref{fig:pairwiseHammingDistances:n100}).

Initially, by the setup of the experiment, all individuals have a Hamming distance of~$0$ to each other.
However, after a short time (relative to the entire runtime), the frequency for distance~$0$ decreases rapidly and the one for distance~$2$ grows rapidly.
From this point on, the frequency for distance~$0$ remains very low for the rest of the run.
That is, the diversity in the population is very high (and remains high).
This supports the findings by Dang et~al.~\cite[Fig.~$2$]{DangFKKLOSS18}, whose experiments also show that the diversity in the population remains very high during a run.

A similar behavior occurs for other neighboring pairs of distances, such as~$2$ and~$4$, although this effect is far more pronounced in Figure~\ref{fig:pairwiseHammingDistances:n1000}.
After the frequency of a distance $d \in \{2, 4, 6\}$ grows, the frequency of distance $d + 2$ also grows until it overtakes the value of~$d$.
From this point on, the frequency of~$d$ typically remains below that of $d + 2$, going so far that~$d$ is almost not present anymore after some time (Figure~\ref{fig:pairwiseHammingDistances:n1000}).
This trend seems to also emerge for distances~$8$ and~$10$, although their frequencies do not meet, since the optimum is found before.
These dynamics suggest that the population does not only diversify in the number of species over time but also increases in pairwise distances to each other.

Last, we see that the run ends in both cases once there is a sufficiently high fraction of individuals with maximum Hamming distance in the population.
This conforms with the hypothesis that the actual expected runtime is in the order of $n\ln(n) + 4^k$.

\section{Conclusion}

In this work, we detected and mathematically proved that the diversity occurring when the \mga optimizes \jump functions persists much longer than known before, namely for a time exponential in the population size $\mu$ instead of quadratic. This result allowed us to prove superior runtime guarantees, in particular, for small population sizes.

Our experiments support our finding that once the population is not anymore dominated by a single genotype, this diversity lasts for a long time. However, our experiments also show that the diversity produced by the \mga is even stronger than what our proofs show. Not only does no genotype class dominate, but also the typical Hamming distance between individuals grows. If such an effect could be proven, this would immediately lead to much stronger runtime guarantees. This is clearly the most interesting, possibly not very easy, continuation of this work.

\bibliographystyle{plain}
\bibliography{ich_master,alles_ea_master,rest}

\begin{thebibliography}{10}

\bibitem{AntipovD20ppsn}
Denis Antipov and Benjamin Doerr.
\newblock Runtime analysis of a heavy-tailed ${(1+(\lambda, \lambda))}$ genetic
  algorithm on jump functions.
\newblock In {\em Parallel Problem Solving From Nature, PPSN 2020, Part~II},
  pages 545--559. Springer, 2020.

\bibitem{AntipovDK19foga}
Denis Antipov, Benjamin Doerr, and Vitalii Karavaev.
\newblock A tight runtime analysis for the ${(1 + (\lambda,\lambda))}$ {GA} on
  {Leading\-Ones}.
\newblock In {\em Foundations of Genetic Algorithms, FOGA 2019}, pages
  169--182. ACM, 2019.

\bibitem{AntipovDK20}
Denis Antipov, Benjamin Doerr, and Vitalii Karavaev.
\newblock The ${(1 + (\lambda,\lambda))}$ {GA} is even faster on multimodal
  problems.
\newblock In {\em Genetic and Evolutionary Computation Conference, GECCO 2020},
  pages 1259--1267. {ACM}, 2020.

\bibitem{AntipovDK22}
Denis Antipov, Benjamin Doerr, and Vitalii Karavaev.
\newblock A rigorous runtime analysis of the ${(1 + (\lambda,\lambda))}$ {GA}
  on jump functions.
\newblock {\em Algorithmica}, 84:1573--1602, 2022.

\bibitem{AugerD11}
Anne Auger and Benjamin Doerr, editors.
\newblock {\em Theory of Randomized Search Heuristics}.
\newblock World Scientific Publishing, 2011.

\bibitem{BianQ22}
Chao Bian and Chao Qian.
\newblock Better running time of the non-dominated sorting genetic
  algorithm~{II} {(NSGA-II)} by using stochastic tournament selection.
\newblock In {\em Parallel Problem Solving From Nature, PPSN 2022}, pages
  428--441. Springer, 2022.

\bibitem{CorusO18tec}
Dogan Corus and Pietro~S. Oliveto.
\newblock Standard steady state genetic algorithms can hillclimb faster than
  mutation-only evolutionary algorithms.
\newblock {\em {IEEE} Transactions on Evolutionary Compututation}, 22:720--732,
  2018.

\bibitem{CorusO20}
Dogan Corus and Pietro~S. Oliveto.
\newblock On the benefits of populations for the exploitation speed of standard
  steady-state genetic algorithms.
\newblock {\em Algorithmica}, 82:3676--3706, 2020.

\bibitem{DangFKKLOSS16}
Duc{-}Cuong Dang, Tobias Friedrich, Timo K{\"{o}}tzing, Martin~S. Krejca,
  Per~Kristian Lehre, Pietro~S. Oliveto, Dirk Sudholt, and Andrew~M. Sutton.
\newblock Escaping local optima with diversity mechanisms and crossover.
\newblock In {\em Genetic and Evolutionary Computation Conference, GECCO 2016},
  pages 645--652. {ACM}, 2016.

\bibitem{DangFKKLOSS18}
Duc{-}Cuong Dang, Tobias Friedrich, Timo K{\"{o}}tzing, Martin~S. Krejca,
  Per~Kristian Lehre, Pietro~S. Oliveto, Dirk Sudholt, and Andrew~M. Sutton.
\newblock Escaping local optima using crossover with emergent diversity.
\newblock {\em {IEEE} Transactions on Evolutionary Computation}, 22:484--497,
  2018.

\bibitem{DangOSS23}
Duc-Cuong Dang, Andre Opris, Bahare Salehi, and Dirk Sudholt.
\newblock A proof that using crossover can guarantee exponential speed-ups in
  evolutionary multi-objective optimisation.
\newblock In {\em Conference on Artificial Intelligence, {AAAI} 2023}. {AAAI}
  Press, 2023.
\newblock To appear.

\bibitem{DebPAM02}
Kalyanmoy Deb, Amrit Pratap, Sameer Agarwal, and T.~Meyarivan.
\newblock A fast and elitist multiobjective genetic algorithm: {NSGA-II}.
\newblock {\em IEEE Transactions on Evolutionary Computation}, 6:182--197,
  2002.

\bibitem{Doerr16}
Benjamin Doerr.
\newblock Optimal parameter settings for the $(1+(\lambda, \lambda))$ genetic
  algorithm.
\newblock In {\em Genetic and Evolutionary Computation Conference, GECCO 2016},
  pages 1107--1114. {ACM}, 2016.

\bibitem{DoerrD15tight}
Benjamin Doerr and Carola Doerr.
\newblock A tight runtime analysis of the (1+({\(\lambda\)}, {\(\lambda\)}))
  genetic algorithm on {OneMax}.
\newblock In {\em Genetic and Evolutionary Computation Conference, GECCO 2015},
  pages 1423--1430. {ACM}, 2015.

\bibitem{DoerrD18}
Benjamin Doerr and Carola Doerr.
\newblock Optimal static and self-adjusting parameter choices for the
  ${(1+(\lambda,\lambda))}$ genetic algorithm.
\newblock {\em Algorithmica}, 80:1658--1709, 2018.

\bibitem{DoerrDE13}
Benjamin Doerr, Carola Doerr, and Franziska Ebel.
\newblock Lessons from the black-box: fast crossover-based genetic algorithms.
\newblock In {\em Genetic and Evolutionary Computation Conference, GECCO 2013},
  pages 781--788. ACM, 2013.

\bibitem{DoerrDE15}
Benjamin Doerr, Carola Doerr, and Franziska Ebel.
\newblock From black-box complexity to designing new genetic algorithms.
\newblock {\em Theoretical Computer Science}, 567:87--104, 2015.

\bibitem{DoerrHK12}
Benjamin Doerr, Edda Happ, and Christian Klein.
\newblock Crossover can provably be useful in evolutionary computation.
\newblock {\em Theoretical Computer Science}, 425:17--33, 2012.

\bibitem{DoerrJKNT13}
Benjamin Doerr, Daniel Johannsen, Timo K{\"o}tzing, Frank Neumann, and
  Madeleine Theile.
\newblock More effective crossover operators for the all-pairs shortest path
  problem.
\newblock {\em Theoretical Computer Science}, 471:12--26, 2013.

\bibitem{DoerrN20}
Benjamin Doerr and Frank Neumann, editors.
\newblock {\em Theory of Evolutionary Computation---Recent Developments in
  Discrete Optimization}.
\newblock Springer, 2020.
\newblock Also available at
  \url{http://www.lix.polytechnique.fr/Labo/Benjamin.Doerr/doerr_neumann_book.html}.

\bibitem{DoerrNS17}
Benjamin Doerr, Frank Neumann, and Andrew~M. Sutton.
\newblock Time complexity analysis of evolutionary algorithms on random
  satisfiable $k$-{CNF} formulas.
\newblock {\em Algorithmica}, 78:561--586, 2017.

\bibitem{DoerrQ23LB}
Benjamin Doerr and Zhongdi Qu.
\newblock From understanding the population dynamics of the {NSGA-II} to the
  first proven lower bounds.
\newblock In {\em Conference on Artificial Intelligence, {AAAI} 2023}. {AAAI}
  Press, 2023.
\newblock To appear.

\bibitem{Holland75}
John~H. Holland.
\newblock {\em Adaptation in Natural and Artificial Systems}.
\newblock University of Michigan Press, 1975.

\bibitem{HuangZCH19}
Zhengxin Huang, Yuren Zhou, Zefeng Chen, and Xiaoyu He.
\newblock Running time analysis of {MOEA/D} with crossover on discrete
  optimization problem.
\newblock In {\em Conference on Artificial Intelligence, {AAAI} 2019}, pages
  2296--2303. {AAAI Press}, 2019.

\bibitem{Jansen13}
Thomas Jansen.
\newblock {\em Analyzing Evolutionary Algorithms -- The Computer Science
  Perspective}.
\newblock Springer, 2013.

\bibitem{JansenW99}
Thomas Jansen and Ingo Wegener.
\newblock On the analysis of evolutionary algorithms -- a proof that crossover
  really can help.
\newblock In {\em European Symposium on Algorithms, {ESA} 1999}, pages
  184--193. Springer, 1999.

\bibitem{JansenW02}
Thomas Jansen and Ingo Wegener.
\newblock The analysis of evolutionary algorithms -- a proof that crossover
  really can help.
\newblock {\em Algorithmica}, 34:47--66, 2002.

\bibitem{Kotzing16}
Timo K{\"{o}}tzing.
\newblock Concentration of first hitting times under additive drift.
\newblock {\em Algorithmica}, 75:490--506, 2016.

\bibitem{KotzingST11}
Timo K{\"{o}}tzing, Dirk Sudholt, and Madeleine Theile.
\newblock How crossover helps in pseudo-{B}oolean optimization.
\newblock In {\em Genetic and Evolutionary Computation Conference, {GECCO}
  2011}, pages 989--996. {ACM}, 2011.

\bibitem{Krejca19}
Martin~S. Krejca.
\newblock {\em Theoretical Analyses of Univariate Estimation-of-Distribution
  Algorithms}.
\newblock PhD thesis, Universit\"at Potsdam, 2019.

\bibitem{LehreY11}
Per~Kristian Lehre and Xin Yao.
\newblock Crossover can be constructive when computing unique input-output
  sequences.
\newblock {\em Soft Computing}, 15:1675--1687, 2011.

\bibitem{MitchellHF93}
Melanie Mitchell, John~H. Holland, and Stephanie Forrest.
\newblock When will a genetic algorithm outperform hill climbing.
\newblock In {\em Advances in Neural Information Processing Systems, {NIPS}
  1993}, pages 51--58. Morgan Kaufmann, 1993.

\bibitem{NeumannORS11}
Frank Neumann, Pietro~S. Oliveto, Günter Rudolph, and Dirk Sudholt.
\newblock On the effectiveness of crossover for migration in parallel
  evolutionary algorithms.
\newblock In {\em Genetic and Evolutionary Computation Conference, GECCO 2011},
  pages 1587--1594. {ACM}, 2011.

\bibitem{NeumannT10}
Frank Neumann and Madeleine Theile.
\newblock How crossover speeds up evolutionary algorithms for the
  multi-criteria all-pairs-shortest-path problem.
\newblock In {\em Parallel Problem Solving from Nature, PPSN 2010, Part {I}},
  pages 667--676. Springer, 2010.

\bibitem{NeumannW10}
Frank Neumann and Carsten Witt.
\newblock {\em Bioinspired Computation in Combinatorial Optimization --
  Algorithms and Their Computational Complexity}.
\newblock Springer, 2010.

\bibitem{OlivetoHY08}
Pietro~S. Oliveto, Jun He, and Xin Yao.
\newblock Analysis of population-based evolutionary algorithms for the vertex
  cover problem.
\newblock In {\em Congress on Evolutionary Computation, {CEC} 2008}, pages
  1563--1570. {IEEE}, 2008.

\bibitem{OlivetoSW22}
Pietro~S. Oliveto, Dirk Sudholt, and Carsten Witt.
\newblock Tight bounds on the expected runtime of a standard steady state
  genetic algorithm.
\newblock {\em Algorithmica}, 84:1603--1658, 2022.

\bibitem{QianYZ13}
Chao Qian, Yang Yu, and Zhi{-}Hua Zhou.
\newblock An analysis on recombination in multi-objective evolutionary
  optimization.
\newblock {\em Artificial Intelligence}, 204:99--119, 2013.

\bibitem{Sudholt12}
Dirk Sudholt.
\newblock Crossover speeds up building-block assembly.
\newblock In {\em Genetic and Evolutionary Computation Conference, {GECCO}
  2012}, pages 689--702. {ACM}, 2012.

\bibitem{Sutton21}
Andrew~M. Sutton.
\newblock Fixed-parameter tractability of crossover: steady-state {GAs} on the
  closest string problem.
\newblock {\em Algorithmica}, 83:1138--1163, 2021.

\end{thebibliography}

}%sloppy, please do not remove

\end{document}